\documentclass{article}

\usepackage[numbers,square,sort&compress]{natbib}
\usepackage[preprint]{neurips_2026}

\usepackage[utf8]{inputenc} 
\usepackage[T1]{fontenc}    
\usepackage{hyperref}       
\usepackage{url}            
\usepackage{booktabs}       
\usepackage{amsfonts}       
\usepackage{nicefrac}       
\usepackage{microtype}      
\usepackage{xcolor}         

\usepackage{amsmath}
\usepackage{amssymb}
\usepackage{graphicx}
\usepackage{algorithm}
\usepackage{algorithmic}
\usepackage{multirow}
\usepackage{subcaption}
\usepackage{bm}
\usepackage{mathtools}
\usepackage{enumitem}
\usepackage{booktabs}
\usepackage{tikz}
\usetikzlibrary{positioning,arrows.meta}

\definecolor{cgray}{HTML}{7A7A7A}
\definecolor{corange}{HTML}{E8743B}
\definecolor{cpurple}{HTML}{7C74B8}

\newcommand{\R}{\mathbb{R}}

\newcommand{\diag}{\operatorname{diag}}

\usepackage{amsthm}
\newtheorem{proposition}{Proposition}

\newtheorem{theorem}{Theorem}
\newtheorem{lemma}{Lemma}

\newtheorem{assumption}{Assumption}

\title{CORA: Per-Slice Coherent Orthogonal Rotation for SVD-based Low-Rank Adaptation}
\author{
  Pengcheng Wang \\
  Purdue University \\
  \texttt{wang4495@purdue.edu} \\
  \And
  Ziran Liu \\
  Shanghai Institute for Mathematics and \\ 
  Interdisciplinary Sciences / Fudan University\\
  \texttt{zliu@simis.cn} \\
  \And
  Wei Wang \\
  Futurewei Technologies, Inc. \\
  \texttt{rickweiwang@futurewei.com} \\
  \And
  Wei Jiang \\
  Futurewei Technologies, Inc. \\
  \texttt{wjiang@futurewei.com} \\
}

\begin{document}

\maketitle
\begin{abstract}
Parameter-Efficient Fine-Tuning (PEFT) commonly adapts pretrained weights through low-rank updates, and recent methods further exploit the singular value decomposition (SVD) of the base weight for initialization or subspace selection. However, these methods do not explicitly preserve the coupled geometry between the pretrained left and right singular bases. Motivated by recent minimum-perturbation theory, which shows that stable finetuning follows a coherent SVD rotation in which a single orthogonal $Q$ acts on both the left singular basis $U_0$ and the right singular basis $V_0$, we prove a per-slice analogue: each row slice of $W_0$ can be adapted by a shared orthogonal rotation $Q_i$ on its left basis $U_i$ and right basis $V_i$ together with a diagonal spectrum shift. We implement this form as \textbf{CORA} (\emph{Coherent Orthogonal Rotation Adaptation}), which applies per-slice orthogonal rotations and a per-layer diagonal scale to the rank-$r$ SVD truncation of $W_0$. CORA uses $\tfrac{1}{2}m(r{-}1)$ trainable parameters per linear layer, about $4{\times}$ fewer than LoRA at the same rank. CORA outperforms LoRA, DoRA, PiSSA, and MiLoRA on commonsense reasoning and code generation while using about $8{\times}$ fewer parameters.
\end{abstract}


\section{Introduction}\label{sec:intro}

%
%

Large Language Models (LLMs) are commonly adapted to downstream tasks through Parameter-Efficient Fine-Tuning (PEFT), since full finetuning is expensive in compute, storage, and deployment. The dominant PEFT method, LoRA~\cite{hu2022lora}, freezes a pretrained weight $W_0 \in \mathbb{R}^{m \times k}$ and learns a low-rank additive update
\[
    \Delta W = BA,
    \qquad
    B \in \mathbb{R}^{m \times r},\ A \in \mathbb{R}^{r \times k},\ r \ll \min(m, k).
\]
This makes adaptation cheap and modular but leaves open which geometric degrees of freedom the adapter should use to modify a pretrained weight.

SVD-based PEFT methods~\cite{meng2024pissa,wang2025milora,lingam2024svft,wang2025kasa} make this choice explicit by decomposing $W_0 = U_0 \Sigma_0 V_0^\top$ into three adaptation axes: the left basis $U_0$, the right basis $V_0$, and the spectrum $\Sigma_0$. They access these axes through additive low-rank updates: PiSSA initializes $A$ and $B$ from the principal singular components~\cite{meng2024pissa}, while MiLoRA updates only the minor components~\cite{wang2025milora}. Orthogonal finetuning methods instead apply a learned orthogonal $R$ as $W = R W_0$ to rotate the basis without altering the spectrum~\cite{qiu2023controlling,liu2024boft}.

Recent theoretical work~\cite{wang2025wsbm} shows that under stability assumptions on a well-pretrained $W_0$, the Frobenius minimum-perturbation finetuning takes the structured form
\[
    \widetilde W
    =
    U_0 Q(\Sigma_0+\Delta\Sigma)Q^\top V_0^\top,
\]
where $Q_U = Q_V = Q$ is a coherent rotation acting on both sides of the spectrum and $\Delta\Sigma$ is a diagonal shift in singular values. Additive SVD-LoRA methods and orthogonal finetuning methods each realize a strict subset of this form: the former bundles $Q$ and $\Delta\Sigma$ jointly inside a rank-$r$ update $\Delta W = BA$, while the latter learns $Q$ alone with $\Delta\Sigma = 0$.

Existing parameterizations realize the coherent rotation--spectrum form only partially. Additive LoRA-style methods can in principle change both basis and spectrum, but these effects are entangled inside the same rank-$r$ update $BA$, so rotation and spectrum share a single parameter budget without a way to separate them. Orthogonal finetuning methods make rotation explicit, but strict orthogonality preserves singular values and cannot express spectral shifts without an additional spectral parameterization.

\begin{figure}[t]
\centering
\includegraphics[width=\linewidth]{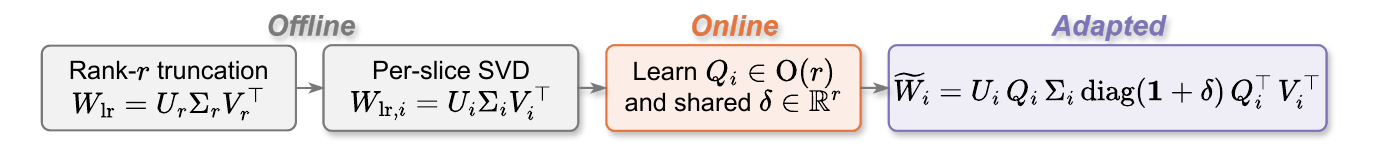}
\caption{\textbf{Overview of CORA.} A weight matrix $W$ is replaced by its rank-$r$ SVD reconstruction $W_{\mathrm{lr}}$, then partitioned into $s = m/r$ row slices. Per-slice SVD factors are precomputed; only the per-slice orthogonal $Q_i$ and a per-layer diagonal scale $\boldsymbol{\delta} \in \mathbb{R}^r$ are learned, with $Q_{U,i} = Q_{V,i} = Q_i$ (Thm.~\ref{thm:slice_coherent}). Per-layer trainable parameters are $\tfrac{1}{2}m(r{-}1)$, roughly $4\times$ fewer than LoRA at the same $r$.}
\label{fig:cora_overview}
\end{figure}

The coherent rotation form is stated for the full weight matrix, while practical adapters operate on structured low-rank matrices to keep the parameter budget small. The global form does not specify how to realize this efficiently.

We address this limitation by extending the coherent rotation to a per-slice adapter. Specifically, we partition a weight matrix into row slices
$W_i \in \mathbb{R}^{r \times k}$ and show that, under mild regularity assumptions inherited from $W_0$, each slice has the analogous
minimum-perturbation form
\[
    W_i^\star
    =
    U_i Q_i(\Sigma_i+\Delta\Sigma_i)Q_i^\top V_i^\top,
    \qquad Q_{U,i}=Q_{V,i}=Q_i,
\]
where $(U_i,\Sigma_i,V_i)$ is the per-slice SVD. We implement this per-slice form as a compact adapter by applying it to the rank-$r$ reconstruction $W_{\mathrm{lr}} = U_{0,r}\Sigma_{0,r}V_{0,r}^\top$ and freezing the residual $W_0 - W_{\mathrm{lr}}$. In practice, we parameterize
$Q_i$ as a block-diagonal closed-form Cayley map of a learnable skew-symmetric matrix, which keeps $Q_i$ exactly orthogonal throughout
training (Section~\ref{sec:theory_implicit}). The resulting method, \textbf{CORA} (Coherent Orthogonal Rotation Adaptation), realizes the coherent rotation form at $\tfrac{1}{2}\,m\,(r{-}1)$ parameters per linear layer, roughly $4{\times}$ fewer than LoRA at the same rank.
We evaluate CORA against SVD-based and general PEFT baselines on three task families: commonsense reasoning (8 tasks), code (HumanEval), and math (GSM8K and MATH), across LLaMA-2-7B, LLaMA-3-8B, and Mistral-7B.
CORA improves the parameter--accuracy trade-off on commonsense reasoning and code generation, exceeding LoRA, DoRA, PiSSA, and MiLoRA at about $8{\times}$ fewer parameters.

Our contributions are:
\begin{itemize}[leftmargin=*, labelindent=0pt, itemindent=0pt, itemsep=5pt, topsep=2pt]
\item We prove a per-slice analog of the coherent rotation theorem of~\cite{wang2025wsbm}: every row slice of $W_0$ follows the same minimum-perturbation form (Theorem~\ref{thm:slice_coherent}), extending it to the per-slice granularity of practical adapters.

\item We parameterize each per-slice rotation as $Q_i = U_i^\top R_i U_i$ for a single learnable orthogonal $R_i$, enforcing the coherent rotation condition $Q_{U,i}^\star = Q_{V,i}^\star$ (Proposition~\ref{prop:R_equiv_shareQ}) with no constraint or penalty during training. CORA applies this parameterization to the rank-$r$ reconstruction $W_{\mathrm{lr}}$ of $W_0$, using $\tfrac{1}{2}m(r{-}1)$ trainable parameters per linear layer, roughly $1/4$ of LoRA at matched rank.

\item CORA reaches $\mathbf{82.16\%}$ on LLaMA-2-7B commonsense at $6.6$\,M parameters, exceeding PiSSA, MiLoRA, DoRA, and LoRA at $\sim 8{\times}$ fewer parameters; the same method transfers to LLaMA-3-8B and Mistral-7B on coding, reaching $\mathbf{48.2}$ and $\mathbf{40.2}$ HumanEval Pass@1 respectively at $10.3$\,M parameters, $4$ to $12\times$ fewer than LoRA, DoRA, LoRI, and DiaBlo.
\end{itemize}

\section{Related work}\label{sec:related}


\textbf{Low-rank and SVD-based adaptation.}
LoRA~\cite{hu2022lora} learns $\Delta W = BA$ with $B \in \R^{m\times r}$ and $A \in \R^{r\times k}$ while freezing $W_0$; close variants augment the update with a learnable magnitude (DoRA~\cite{liu2024dora}), adaptive per-layer rank (AdaLoRA~\cite{zhang2023adalora}), or cross-layer parameter sharing (VeRA~\cite{kopiczko2024vera}).
A growing subfamily uses the pretrained SVD $W_0 = U_0 \Sigma_0 V_0^\top$ as the design principle: PiSSA~\cite{meng2024pissa} and MiLoRA~\cite{wang2025milora} initialize $A, B$ from the principal and minor singular components respectively, CorDA~\cite{yang2024corda} uses an activation-weighted decomposition $W_0 X X^\top$ for task-aware initialization, SVFT~\cite{lingam2024svft} freezes $U_0, V_0$ and trains a sparse coupling matrix $M$ so that $\Delta W = U_0 M V_0^\top$, KaSA~\cite{wang2025kasa} adds knowledge-aware gating on singular values, and LoftQ~\cite{li2024loftq} combines SVD initialization with 4-bit quantization.
Optimization-side variants such as LoRA-GA~\cite{wang2024loraga} and LoRA-Pro~\cite{wang2025lorapro} improve initialization and gradient scaling for the same $\Delta W = BA$ parameterization.
All of these methods bias the optimizer toward the pretrained spectrum but leave the additive $\Delta W = BA$ block bundling rotation and spectrum adaptation, without enforcing the coherent rotation condition $Q_U = Q_V$ established in~\cite{wang2025wsbm} (see Section~\ref{sec:prelim_coherent}).

\textbf{Orthogonal and block-diagonal finetuning.}
Orthogonal finetuning methods parameterize the update as a rotation of $W_0$: OFT~\cite{qiu2023controlling} uses a strictly orthogonal block-diagonal $R$ applied as $\widetilde W = R\,W_0$; BOFT~\cite{liu2024boft} replaces the block structure with a butterfly factorization; OFTv2~\cite{qiu2025oftv2} moves the rotation to the activation side for throughput; POET~\cite{qiu2025poet} extends to a two-sided orthogonal equivalence $\widetilde W = P W_0 Q^\top$; and PSOFT~\cite{wu2025psoft} constrains the rotation to the principal subspace of $W_0$. These methods operate at substantially higher training cost than the SVD-initialized PEFT family and target distinct deployment regimes. We use them as conceptual context and benchmark CORA against the SVD-initialized PEFT family in Section~\ref{sec:setup}.
DiaBlo~\cite{gurses2025diablo} replaces the low-rank product $\Delta W = BA$ with a direct block-diagonal update on $W$.
All of these methods apply at the full-matrix scale; none exploits a per-slice decomposition.
\textbf{CORA} is the per-slice specialization of this family, with $s = m/r$ independent orthogonal rotations, one per row slice of the rank-$r$ reconstruction $W_{\mathrm{lr}} = U_{0,r}\Sigma_{0,r}V_{0,r}^\top$, unifying the additive SVD-initialized and orthogonal-rotation families under a single form~\eqref{eq:coherent_form}.

\section{Preliminaries}\label{sec:spectral_analysis}

\subsection{Notation}\label{sec:prelim_notation}

Let $W_0 \in \mathbb{R}^{m \times k}$ be a pretrained weight matrix with SVD $W_0 = U_0 \Sigma_0 V_0^\top$, where $U_0 \in \mathbb{R}^{m \times d}$, $V_0 \in \mathbb{R}^{k \times d}$, $\Sigma_0=\mathrm{diag}(\sigma_1,\ldots,\sigma_d)$, $d=\min(m,k)$, and $\sigma_1 \geq \cdots \geq \sigma_d > 0$. Write $W_{\mathrm{lr}} := U_{0,r}\Sigma_{0,r}V_{0,r}^\top$ for the rank-$r$ truncation, where $U_{0,r}$, $\Sigma_{0,r}$, $V_{0,r}$ keep the leading $r$ singular components.

\subsection{Minimum-perturbation form}
\label{sec:prelim_coherent}

We first recall the global minimum-perturbation form for SVD-aware
finetuning. Under stability assumptions on a well-pretrained $W_0$ and the
Frobenius minimum-perturbation objective, the finetuned weight admits a
coherent in-basis rotation, which we restate in our notation.

\begin{theorem}[{Coherent rotation form, after~\cite[Prop.~9.4]{wang2025wsbm}}]
\label{thm:global_coherent}
Let $W_0 \in \mathbb{R}^{m \times k}$ be a pretrained weight matrix with
SVD $W_0 = U_0\Sigma_0V_0^\top$. Under the standard stability assumptions and the
Frobenius minimum-perturbation objective, the finetuned weight has the form
\begin{equation}\label{eq:coherent_form}
    W^\star
    =
    U_0 Q(\Sigma_0+\Delta\Sigma)Q^\top V_0^\top,
\end{equation}
where $Q$ is orthogonal and $\Delta\Sigma$ is diagonal. Equivalently, the left and right in-basis rotations satisfy $Q_U^\star = Q_V^\star = Q$.
\end{theorem}
Among all finetuned weights that achieve a given task objective, the one with smallest $\|\Delta W\|_F$ keeps the pretrained singular coordinate system fixed and acts on $\Sigma_0$ by a similarity transform $Q\,(\cdot)\,Q^\top$ in that basis. Any mismatch
between the left and right in-basis rotations, i.e., $Q_U \neq Q_V$, breaks
this coherent structure and increases the perturbation needed to realize the
same adaptation.

Eq.~\eqref{eq:coherent_form} thus exposes two coupled degrees of freedom: a spectral shift $\Delta\Sigma$ and a single coherent rotation $Q$ that acts on both sides of the spectrum. Two existing PEFT families realize this structure only partially: additive SVD-LoRA methods bundle spectrum and basis adaptation inside $\Delta W = BA$, while orthogonal finetuning methods learn $Q$ alone with $\Delta\Sigma = 0$.

This form is stated for the full matrix: $Q$ acts on $W_0$'s entire singular
bases. Section~\ref{sec:theory} extends it to the row-slice granularity
that practical adapters use.

\subsection{Per-slice decomposition}
\label{sec:prelim_slice}

We work with a per-slice decomposition of the source matrix. Given $W \in \mathbb{R}^{m \times k}$, we partition its rows into $s = m/r$ contiguous slices $W_i \in \mathbb{R}^{r \times k}$, so that $W = [W_1^\top \cdots W_s^\top]^\top$, and write each slice's SVD as $W_i = U_i \Sigma_i V_i^\top$. CORA applies per-slice adaptation to the rank-$r$ reconstruction $W_{\mathrm{lr}} = U_{0,r}\Sigma_{0,r}V_{0,r}^\top$ while freezing the residual $W_0 - W_{\mathrm{lr}}$. The rank-$r$ truncation keeps the per-slice SVD cache compact and acts as a spectral regularizer on the source matrix.

\section{Method}
\label{sec:theoryandmethod}

We first extend the coherent rotation form (Theorem~\ref{thm:global_coherent})
from a full weight matrix to row-slice submatrices
(Section~\ref{sec:theory}). We then introduce a \emph{rotation reparameterization}
that realizes the resulting per-slice form with a single learnable matrix per
slice (Section~\ref{sec:theory_implicit}), and finally instantiate this form
as the CORA adapter, with and without a learnable spectrum scale
(Section~\ref{sec:method_overview}).
See Appendix~\ref{app:proof_slice_coherent}--\ref{app:hyperparams} for full proofs and implementation details (Cayley map, parameter accounting); the main text states results and sketches the justifications.

\subsection{Per-slice coherent rotation form}
\label{sec:theory}

Theorem~\ref{thm:global_coherent} is stated for the full pretrained
matrix $W_0$. We now state its row-slice generalization.

\paragraph{Setup.}
Following the per-slice decomposition of Section~\ref{sec:prelim_slice}, let $W$ denote either $W_0$ or its rank-$r$ reconstruction $W_{\mathrm{lr}} = U_{0,r}\Sigma_{0,r}V_{0,r}^\top$, partitioned into $s = m/r$ row slices $W_i \in \mathbb{R}^{r \times k}$ with SVDs $W_i = U_i \Sigma_i V_i^\top$ ($U_i \in \mathbb{R}^{r \times r}$ square orthogonal, $\Sigma_i$ diagonal, $V_i \in \mathbb{R}^{k \times r}$ column-orthonormal).

\begin{assumption}[Slice-level regularity]
\label{assump:slice_regularity}
Each slice $W_i$ has non-degenerate spectrum
($\sigma_{\min}(W_{i}) > 0$), bounded condition number
($\kappa(W_{i}) < \infty$), and
small stable perturbation ($\|\Delta W_{i}\|_F \ll \|W_{i}\|_F$).
\end{assumption}
Assumption~\ref{assump:slice_regularity} requires each slice to inherit the regularity
of the parent matrix; it holds generically because row-slicing preserves the
column-space rank, and the singular-value interlacing
theorem~\cite{stewart1990matrix} ensures
$\sigma_{\min}(W_{i}) \geq \sigma_{\min}(W) > 0$ for generic slice boundaries.

\begin{theorem}[Per-slice coherent rotation form]
\label{thm:slice_coherent}
Let $W_i = U_i\Sigma_iV_i^\top$ be a row slice satisfying
Assumption~\ref{assump:slice_regularity}. Under the slice-level Frobenius
minimum-perturbation objective, the finetuned slice has the form
\begin{equation}\label{eq:slice_coherent_form}
    W_i^\star
    =
    U_i Q_i (\Sigma_i + \Delta\Sigma_i) Q_i^\top V_i^\top,
\end{equation}
where $Q_i \in \mathrm{O}(r)$ is orthogonal and $\Delta\Sigma_i$ is diagonal.
Equivalently, the left and right per-slice in-basis rotations satisfy
\(Q_{U,i}^\star = Q_{V,i}^\star = Q_i\).
\end{theorem}

\paragraph{Proof sketch.}
The theorem is a slice-local statement: a single slice is analyzed under
its own Frobenius minimum-perturbation objective, lifting the global
argument to local SVD coordinates. Two ingredients enable this lift.
(i)~The Frobenius norm is row-partition additive,
$\|\Delta W\|_F^2 = \sum_{i=1}^{s} \|\Delta W_i\|_F^2$, so the per-slice
cost $\|\Delta W_i\|_F$ is the natural local objective.
(ii)~Each slice inherits the regularity of $W$ via singular-value
interlacing (Assumption~\ref{assump:slice_regularity}), so the local
problem is well-posed in $(U_i, \Sigma_i, V_i)$. Within each slice, the
spectrum-fixed Frobenius-distance minimization has a unitary similarity
solution by classical Schur--Horn / Von Neumann
arguments~\cite{horn2012matrix}. The full proof is in
Appendix~\ref{app:proof_slice_coherent}.

\paragraph{Source matrix.}
Theorem~\ref{thm:slice_coherent} applies whenever the source matrix $W$ is
square orthogonal in its column basis. CORA takes $W = W_{\mathrm{lr}} = U_{0,r}\Sigma_{0,r}V_{0,r}^\top$
and freezes the residual $W_0 - W_{\mathrm{lr}}$. The rank-$r$ truncation
removes tail singular directions and acts as a spectral regularizer.

\subsection{Rotation reparameterization}
\label{sec:theory_implicit}

Theorem~\ref{thm:slice_coherent} specifies the desired per-slice form, but a
direct implementation would require parameterizing each $Q_i$ on the orthogonal
manifold $\mathrm{O}(r)$, which involves expensive retraction or projection
steps. We instead reparameterize each $Q_i$ through an unconstrained learnable
matrix $R_i$ per slice, related by a fixed basis-conjugation in the per-slice
SVD frame. This is analogous to weight-normalization-style
reparameterizations~\cite{salimans2016weight}: the learned $R_i$ lives
in an unconstrained ambient space, while the target rotation $Q_i$ is recovered
through a fixed map.

Specifically, for each slice, let $R_i \in \mathbb{R}^{r \times r}$ be a learnable matrix
(orthogonal by default; see Section~\ref{sec:method_overview}), and
define $Q_i := U_i^\top R_i U_i$ (a similarity transform of $R_i$ by the
per-slice left singular basis $U_i$).
Substituting into Eq.~\eqref{eq:slice_coherent_form} and using the identity
$U_i U_i^\top = I_r$ (since $U_i$ is square orthogonal),
$U_i Q_i = (U_i U_i^\top) R_i U_i = R_i U_i$, which gives the equivalent reparameterized form:
\begin{equation}\label{eq:cora_form}
    \widetilde W_i
    \;=\; R_i \, U_i \, \Sigma_i \, \mathrm{diag}(\mathbf{1} + \boldsymbol{\delta}) \, Q_i^\top \, V_i^\top,
    \qquad
    Q_i^\top = U_i^\top R_i^\top U_i,
\end{equation}
where $\boldsymbol{\delta} \in \mathbb{R}^r$ is a per-layer scale vector applied to every slice $i$ of the layer. Only $R_i$ per slice (and optionally the shared $\boldsymbol{\delta}$ per layer) is learned; the $Q_i^\top$ factor is computed from $R_i$ at evaluation time.

\begin{proposition}[Rotation reparameterization realizes the coherent rotation form]
\label{prop:R_equiv_shareQ}
Let $W_i = U_i\Sigma_iV_i^\top$ be a slice with $U_i \in \mathbb{R}^{r\times r}$
square orthogonal, $\Sigma_i$ diagonal, and $V_i \in \mathbb{R}^{k\times r}$
column-orthonormal. Let $R_i \in \mathrm{O}(r)$ be orthogonal and
$\boldsymbol{\delta} \in \mathbb{R}^r$ a scale vector. Define
$Q_i = U_i^\top R_i U_i$ and form $\widetilde W_i$ via Eq.~\eqref{eq:cora_form}. Then:
\begin{enumerate}
    \item[(i)] $\widetilde W_i = U_i Q_i \Sigma_i\,\mathrm{diag}(\mathbf{1}+\boldsymbol{\delta})\, Q_i^\top V_i^\top$;
    in particular $Q_i$ is orthogonal.
    \item[(ii)] The SVD of $\widetilde W_i$ satisfies the coherent rotation
    condition $Q_{U,i}^\star = Q_{V,i}^\star$.
\end{enumerate}
\end{proposition}

\paragraph{Proof sketch.}
Part (i) follows by the substitution above and the identity
$U_i U_i^\top = I_r$. Part (ii) follows by writing $\widetilde W_i = B_i C_i$
with $B_i = R_i (U_i \Sigma_i\,\mathrm{diag}(\mathbf{1}+\boldsymbol{\delta})\, U_i^\top) R_i^\top$ symmetric
and $C_i = U_i V_i^\top$ row-orthonormal. The SVD of
$\widetilde W_i$ is then determined by the eigendecomposition of $B_i$, which
induces the same in-basis rotation on the left and right.
The full proof is in Appendix~\ref{app:proof_one_sided}.

\paragraph{Remark.}
Proposition~\ref{prop:R_equiv_shareQ} reduces the per-slice parameterization to a single orthogonal $R_i \in \mathrm{O}(r)$ instead of parameterizing $Q_{U,i}$ and $Q_{V,i}$ separately. The coherent rotation condition $Q_{U,i}^\star = Q_{V,i}^\star$ holds identically.
The shared scale $\boldsymbol{\delta} \in \mathbb{R}^r$ is an optional component, which scales the per-slice singular values entrywise by $\mathbf{1}+\boldsymbol{\delta}$. When $\boldsymbol{\delta}$ is omitted, the per-slice spectrum stays at $\Sigma_i$.

\subsection{CORA adapter}
\label{sec:method_overview}

Given a pretrained linear layer $y = W_0 x$, CORA adapts it to
$\widetilde y = \widetilde W x$ using the per-slice parameterization above.
CORA has three offline-precomputable components and two
learnable components.

\paragraph{Offline components.}
For each adapted layer, we precompute (i) the SVD $W_0 = U_0 \Sigma_0 V_0^\top$, (ii) the source matrix $W_{\mathrm{lr}} = U_{0,r}\Sigma_{0,r}V_{0,r}^\top$ obtained by truncating to the top-$r$ singular components, and (iii) the per-slice SVD factors $(U_i, \Sigma_i, V_i)$ obtained by partitioning $W_{\mathrm{lr}}$ into $s = m/r$ row slices.

\paragraph{Learnable parameters.}
CORA learns (i) per slice, a block-diagonal orthogonal rotation $R_i = \operatorname{blkdiag}(\phi(A_i^{(1)}),\ldots,\phi(A_i^{(r/b)}))$, where each $A_i^{(g)} \in \mathbb{R}^{b \times b}$ is skew-symmetric, $b$ is the inner block size with $b \mid r$, and $\phi:\mathrm{skew}(b) \to \mathrm{O}(b)$ is a Cayley-type map (Appendix~\ref{app:cayley_impl}); and (ii) optionally, per layer, a single shared scale vector $\boldsymbol{\delta} \in \mathbb{R}^r$ applied multiplicatively as $\Sigma_i \mapsto \Sigma_i\,\mathrm{diag}(\mathbf{1}+\boldsymbol{\delta})$ to every slice $i$ of that layer. Tying $\boldsymbol{\delta}$ across slices reduces the spectrum-correction cost from $m$ to $r$ scalars per layer while preserving layer-wide singular-value structure.
By default, CORA learns the per-slice rotation $R_i$ together with the shared scale $\boldsymbol{\delta}$, a configuration we denote \textbf{CORA+$\boldsymbol{\delta}$} and use throughout our main experiments. Setting $\boldsymbol{\delta} = 0$ recovers the rotation-only configuration, which we refer to as plain \textbf{CORA}. Both configurations
satisfy the coherent rotation form by Proposition~\ref{prop:R_equiv_shareQ}, and
they differ only in whether the spectrum is allowed to scale.

\paragraph{Parameter accounting.}
With the default full-block setting ($b = r$), each slice's Cayley generator
$A_i \in \R^{r \times r}$ is skew-symmetric with $r(r-1)/2$ free parameters.
Summing over the $m/r$ slices of a weight matrix $W \in \R^{m \times k}$,
\begin{equation}\label{eq:cora_params}
    |\theta_{\text{CORA}}^{(W)}| \;=\; \tfrac{1}{2}\, m\,(r-1),
\end{equation}
independent of the column dimension $k$ (the optional shared scale adds $r$
scalars per layer). LoRA at the same rank $r$ uses
$|\theta_{\text{LoRA}}^{(W)}| = r(m+k)$, so for a large rank $r$
\begin{equation}\label{eq:cora_lora_ratio}
    \frac{|\theta_{\text{CORA}}^{(W)}|}{|\theta_{\text{LoRA}}^{(W)}|}
    \;\xrightarrow{r \to \infty}\; \frac{m}{2(m+k)},
\end{equation}
which equals exactly $1/4$ for square projections ($m = k$, e.g., the attention query/key/value layers in LLaMA-2-7B) and remains close to $1/4$ on the weighted average across LLaMA-2-7B's full target set. Per-tier counts ($r \in \{16, 32, 64, 128\}$ giving $6.6$, $13.6$, $27.6$, $55.7$\,M parameters) are reported in Table~\ref{tab:main_per_task}.

\paragraph{Implementation summary.}
At inference, the adapted weight is computed slice by slice via
Eq.~\eqref{eq:cora_form} and reassembled. We instantiate
$\phi$ as the \emph{closed-form Cayley map} $\phi(A) = (I - A)(I + A)^{-1}$,
which keeps each block $\phi(A_i^{(g)})$ exactly orthogonal and imposes no
convergence-radius constraint on $A_i^{(g)}$, so no auxiliary regularizer is needed. Implementation details and parameter counts are summarized in Appendix~\ref{app:cayley_impl} and Section~\ref{sec:experiments}.

\section{Experiments}\label{sec:experiments}

\subsection{Setup}\label{sec:setup}

\paragraph{Models, data, and metrics.}
We evaluate CORA across three task families. For commonsense reasoning on LLaMA-2-7B and LLaMA-3-8B, we train on commonsense\_170k~\cite{hu2023llmadapters} and report arithmetic-mean accuracy over the eight standard benchmarks (BoolQ, PIQA, SIQA, HellaSwag, WinoGrande, ARC-Easy, ARC-Challenge, OpenBookQA), following the protocol used by DoRA, MiLoRA, and DiaBlo~\cite{liu2024dora,wang2025milora,gurses2025diablo}. For mathematical reasoning on LLaMA-2-7B, we train on MetaMathQA-395K~\cite{yu2024metamath} and evaluate on GSM8K~\cite{cobbe2021gsm8k} and MATH~\cite{hendrycks2021math} via Exact Match. For code generation on LLaMA-3-8B and Mistral-7B, we train on CodeAlpaca~\cite{chaudhary2023codealpaca} and report HumanEval Pass@1~\cite{chen2021humaneval}. We adapt the attention query/key/value and FFN up/down projections on LLaMA-2-7B, extending to the full set (adding the attention output and FFN gate projections) on LLaMA-3-8B and Mistral-7B, matching DoRA~\cite{liu2024dora}'s target modules.

\paragraph{Baselines.}
Our primary baselines are the SVD-based PEFT methods that publish on the corresponding benchmark: PiSSA~\cite{meng2024pissa} and MiLoRA~\cite{wang2025milora} on commonsense and math, plus LoRA~\cite{hu2022lora} and DoRA~\cite{liu2024dora} as architectural reference. For code generation, because HumanEval results are not reported for the SVD-based baselines, we additionally include LoRI~\cite{zhang2025lori} and DiaBlo~\cite{gurses2025diablo}. 
We take baseline numbers from the original publications and from our HuggingFace-Trainer reproductions, and note the source for each row in the corresponding table caption. See Appendix~\ref{app:hyperparams} for training hyperparameters.

\subsection{Main results on commonsense reasoning}\label{sec:results_main}

Table~\ref{tab:main_per_task} reports per-task accuracy on LLaMA-2-7B and LLaMA-3-8B at four parameter tiers ($r \in \{8, 16, 32, 64\}$). CORA cells use the default variant (closed-form Cayley with shared scale $\boldsymbol{\delta}$; see Table~\ref{tab:sub_variant_ablation}).

\begin{table}[h!]
\centering
\caption{Per-task commonsense accuracy on LLaMA-2-7B and LLaMA-3-8B (commonsense\_170k, 8-task protocol). Baselines: LoRA / DoRA from~\cite{liu2024dora}; PiSSA / MiLoRA from~\cite{wang2025milora}; Full FT from~\cite{gurses2025diablo}. \underline{Underlined}: SVD-family. Best in \textbf{bold} (Full FT excluded).}
\label{tab:main_per_task}
\small
\setlength{\tabcolsep}{2pt}
\begin{tabular}{@{}l|lr|rrrrrrrr|r@{}}
\toprule
Method & $r/N$ & \#Params & BoolQ & PIQA & SIQA & HellaS & WinoG & ARC-e & ARC-c & OBQA & AVG \\
\midrule
\multicolumn{12}{l}{\textit{LLaMA-2-7B}} \\
\midrule
Full FT      & N/A     & 6.7\,B (100\%)   & 73.3 & 85.7 & 81.0 & 90.2 & 86.9 & 88.6 & 77.4 & 85.2 & 83.5 \\
LoRA         & $r{=}32$ & 56\,M (0.84\%)  & 69.8 & 79.9 & 79.5 & 83.6 & 82.6 & 79.8 & 64.7 & 81.0 & 77.6 \\
DoRA         & $r{=}32$ & 57\,M (0.85\%)  & 71.8 & 83.7 & 76.0 & \textbf{89.1} & 82.6 & 83.7 & 68.2 & 82.4 & 79.7 \\
\underline{PiSSA}        & $r{=}32$ & 56\,M (0.84\%)  & 67.6 & 78.1 & 78.4 & 76.6 & 78.0 & 75.8 & 60.2 & 75.6 & 73.8 \\
\underline{MiLoRA}       & $r{=}32$ & 56\,M (0.84\%)  & 67.6 & 83.8 & 80.1 & 88.2 & 82.0 & 82.8 & 68.8 & 80.6 & 79.2 \\
\textbf{CORA} (ours) & $r{=}64$ & 27.6\,M (0.41\%) & 70.8 & 83.9 & 80.2 & 81.3 & 84.5 & 85.8 & 71.1 & 82.0 & 79.95 \\
\textbf{CORA} (ours) & $r{=}32$ & 13.6\,M (0.20\%) & 72.2 & \textbf{85.5} & 81.5 & 87.2 & \textbf{86.7} & 86.4 & 72.8 & 84.2 & 82.06 \\
\textbf{CORA} (ours) & $r{=}16$ &  6.6\,M (0.10\%) & \textbf{73.2} & 84.6 & \textbf{81.8} & 88.6 & 84.9 & \textbf{87.1} & \textbf{73.5} & 83.6 & \textbf{82.16} \\
\textbf{CORA} (ours) & $r{=}8$  &  3.1\,M (0.05\%) & 69.7 & 84.2 & 81.3 & 88.5 & 85.2 & \textbf{87.1} & 73.1 & \textbf{85.2} & 81.79 \\
\midrule
\multicolumn{12}{l}{\textit{LLaMA-3-8B}} \\
\midrule
Full FT      & N/A     & 8\,B (100\%)     & 76.4 & 89.7 & 82.5 & 95.5 & 89.6 & 92.9 & 84.3 & 89.2 & 87.5 \\
LoRA         & $r{=}32$ & 63\,M (0.79\%)  & 70.8 & 85.2 & 79.9 & 91.7 & 84.3 & 84.2 & 71.2 & 79.0 & 80.8 \\
DoRA         & $r{=}32$ & 63\,M (0.79\%)  & \textbf{74.6} & 89.3 & 79.9 & 95.5 & 85.6 & 90.5 & 80.4 & 85.8 & 85.2 \\
\underline{PiSSA}        & $r{=}32$ & 63\,M (0.79\%)  & 67.1 & 81.1 & 77.2 & 83.6 & 78.9 & 77.7 & 63.2 & 74.6 & 75.4 \\
\underline{MiLoRA}       & $r{=}32$ & 63\,M (0.79\%)  & 68.8 & 86.7 & 77.2 & 92.9 & 85.6 & 86.8 & 75.5 & 81.8 & 81.9 \\
\textbf{CORA} (ours) & $r{=}64$ & 43.3\,M (0.54\%) & 72.2 & 88.2 & 82.7 & 96.1 & 88.2 & 92.2 & 83.1 & 87.0 & 86.21 \\
\textbf{CORA} (ours) & $r{=}32$ & 21.3\,M (0.27\%) & 72.0 & 89.2 & \textbf{83.2} & 96.2 & \textbf{88.5} & \textbf{93.5} & 83.2 & \textbf{88.2} & \textbf{86.75} \\
\textbf{CORA} (ours) & $r{=}16$ & 10.3\,M (0.13\%) & 72.9 & \textbf{89.9} & 82.6 & \textbf{96.3} & 88.2 & 92.6 & \textbf{83.5} & 87.2 & 86.65 \\
\textbf{CORA} (ours) & $r{=}8$  &  4.8\,M (0.06\%) & 72.5 & 89.0 & 81.1 & 95.6 & 87.6 & 93.1 & 82.1 & 87.0 & 86.00 \\
\bottomrule
\end{tabular}
\end{table}

CORA improves over the SVD-based PEFT family at every parameter tier on both models. On LLaMA-2-7B, CORA $r{=}16$ at $6.6$\,M parameters reaches $82.16\%$, $8.4$ points above PiSSA $r{=}32$ ($56$\,M) and $3.0$ points above MiLoRA $r{=}32$ ($56$\,M). On LLaMA-3-8B, CORA $r{=}16$ at $10.3$\,M parameters reaches $86.65\%$, $11.3$ points above PiSSA and $4.8$ points above MiLoRA, again at roughly $6{\times}$ fewer parameters. PiSSA and MiLoRA constrain $A$ and $B$ but still adapt through an additive low-rank update; this update cannot in general realize the per-slice coherent rotation that CORA derives in closed form (Theorem~\ref{thm:slice_coherent}).

Against the LoRA family, CORA matches or exceeds the published numbers at \(4\times\) to \(9\times\) fewer parameters. CORA $r{=}16$ exceeds DoRA $r{=}32$ by $2.5$ points on LLaMA-2-7B and $1.5$ points on LLaMA-3-8B. At the smallest budget, CORA $r{=}8$ uses $3.1$\,M parameters on LLaMA-2-7B ($81.79\%$) and $4.8$\,M on LLaMA-3-8B ($86.00\%$), between \(5\%\) and \(8\%\) of LoRA $r{=}32$'s budget  while remaining above LoRA $r{=}32$ on both models.


\subsection{Beyond commonsense: code and math}\label{sec:results_math_code}

We evaluate CORA on two additional task families to test transfer beyond commonsense. For code generation we finetune LLaMA-3-8B and Mistral-7B on CodeAlpaca and report HumanEval Pass@1 in Table~\ref{tab:code}, with reference baselines from DiaBlo~\cite{gurses2025diablo} Tab.~3 and LoRI~\cite{zhang2025lori} Tab.~2. For mathematical reasoning we finetune LLaMA-2-7B on MetaMathQA-395K and report GSM8K and MATH accuracy in Table~\ref{tab:math_l2}, with reference baselines drawn from MiLoRA~\cite{wang2025milora} Tab.~2 and DiaBlo~\cite{gurses2025diablo} Tab.~2. We use the same hyperparameters as the commonsense setup (Appendix~\ref{app:hyperparams}).

\begin{table}[h]
\centering
\caption{Code generation on LLaMA-3-8B and Mistral-7B (CodeAlpaca, HumanEval Pass@1). Baselines from DiaBlo~\cite{gurses2025diablo} and LoRI~\cite{zhang2025lori}. Best in \textbf{bold}.}
\label{tab:code}
\small
\begin{tabular}{lcc cc}
\toprule
& \multicolumn{2}{c}{LLaMA-3-8B} & \multicolumn{2}{c}{Mistral-7B} \\
\cmidrule(lr){2-3} \cmidrule(lr){4-5}
Config & Params & HumanEval & Params & HumanEval \\
\midrule
LoRA   $r{=}32$    (ref) &  90\,M & 34.7 &  91\,M & 33.8 \\
DoRA   $r{=}32$    (ref) &  90\,M & 33.1 &  91\,M & 33.7 \\
LoRI   $r{=}32$    (ref) &  45\,M & 43.2 &  46\,M & 33.8 \\
DiaBlo $N{=}128$   (ref) &  61\,M & 39.4 &  61\,M & 34.0 \\
DiaBlo $N{=}64$    (ref) & 121\,M & 43.2 & 122\,M & 34.4 \\
\midrule
CORA $r{=}64$ (ours)  & 43.3\,M & 47.0           & 43.3\,M & \textbf{40.9} \\
CORA $r{=}32$ (ours)  & 21.3\,M & 45.1           & 21.3\,M & \textbf{40.9} \\
CORA $r{=}16$ (ours)  & 10.3\,M & \textbf{48.2}  & 10.3\,M & 40.2 \\
CORA $r{=}8$  (ours)  &  4.8\,M & 45.1           &  4.8\,M & 40.2 \\
\bottomrule
\end{tabular}
\end{table}

On code, CORA $r{=}16$ at $10.3$\,M reaches $48.2$ HumanEval Pass@1 on LLaMA-3-8B, $5.0$ points above DiaBlo $N{=}64$ at $121$\,M and LoRI $r{=}32$ at $45$\,M. On Mistral-7B, CORA $r{=}32$ and $r{=}64$ reach $40.9$, $\sim 6$--$7$ points above LoRA, DoRA, and the DiaBlo variants at $\sim 60$--$120$\,M; CORA $r{=}16$ at $10.3$\,M retains $40.2$, and CORA $r{=}8$ at $4.8$\,M retains $45.1$/$40.2$ on the two models. The lower-rank tiers degrade slightly faster on Mistral-7B than on the two LLaMA models, an effect we tentatively attribute to its grouped-query attention layout, which couples the K/V projections across heads and likely concentrates more task-relevant signal into the small fraction of slices each $R_i$ adapts.

\begin{table}[h]
\centering
\caption{Mathematical reasoning on LLaMA-2-7B (GSM8K + MATH). 
Best in \textbf{bold}.}
\label{tab:math_l2}
\small
\begin{tabular}{lccc}
\toprule
Config & Params & GSM8K & MATH \\
\midrule
Full FT (reference)  & 6.74B & 66.5 & 19.8 \\
\midrule
LoRA   $r{=}64$       & 113\,M  & 60.6 & 16.9 \\
PiSSA  $r{=}64$       & 113\,M  & 58.2 & 15.8 \\
MiLoRA $r{=}64$       & 113\,M  & 63.5 & \textbf{17.8} \\
\midrule
CORA $r{=}128$ ($W_0$)  & 55.7\,M & 62.4 & 14.4 \\
CORA $r{=}64$ ($W_{lr}$)  & 27.6\,M & 62.0 & 12.7 \\
\bottomrule
\end{tabular}
\end{table}

On math, CORA $r{=}64$ at $27.6$\,M reaches $62.0$ GSM8K, $1.4$ points above LoRA $r{=}64$ at $113$\,M; on MATH, CORA $r{=}64$ reaches $12.7$, below MiLoRA $r{=}64$'s $17.8$. 

Overall, principal-component methods tend to underperform on MATH. We hypothesize that this gap reflects the \emph{spectral locality} of mathematical knowledge in the pretrained backbone. If mathematical reasoning is underrepresented in general-purpose pretraining relative to natural-language tasks (a common assumption motivating math-specific instruction-tuning corpora such as MetaMathQA~\cite{yu2024metamath}), then the principal singular directions of $W_0$ should encode comparatively little MATH-relevant structure. Adapters that operate within the top-$r$ principal subspace (PiSSA and the default CORA that applies the rotation to $W_{\mathrm{lr}}$ and discards the tail singular directions on which MATH relies) inherit this limitation directly. MiLoRA, which freezes the principal components and trains the minor ones, has access to precisely the residual directions where math-specific structure is more likely to concentrate, consistent with its own framing of minor components as task-adaptive~\cite{wang2025milora}. DiaBlo, which adapts entries of $W_0$ directly, also performs well on MATH at higher parameter budgets ($20.4$ at $N{=}32$, $141$\,M).

To test this hypothesis, in Section~\ref{sec:results_ablation}, we evaluate a variant of CORA using the full \(W_0\) as the source matrix instead of \(W_{lr}\), which retains these tail directions and is able to close part of the performance gap. This indicates that further studies can be conducted to improve MATH. For example, a controlled comparison that varies only the source spectrum within CORA’s slice-wise framework (e.g., rotating the bottom-\(r\) subspace
as a direct counterpart to MiLoRA) would isolate the effect, and we leave this to future work.

\subsection{Ablations}\label{sec:results_ablation}

\paragraph{Shared spectrum scale.}
We ablate whether the shared spectrum scale $\boldsymbol{\delta}$ is enabled (Section~\ref{sec:method_overview}). Table~\ref{tab:sub_variant_ablation} reports the LLaMA-2-7B 8-task average across the four CORA tiers. Closed-form solve Cayley is used throughout. The shared scale gains $0.8$ to $2.4$ points at $r \in \{16, 32, 64\}$ and rescues the $r{=}8$ configuration from collapse ($+13.8$ points). Therefore $\boldsymbol{\delta}$ is enabled as the default.

\paragraph{Rotation sharing across slice pairs.}
By default, CORA learns one orthogonal $R_i$ per slice. We can halve the rotation parameter count by sharing one $R_i$ across pairs of consecutive slices, the \emph{paired} variant. The relevant comparison is at matched parameter budget rather than at matched rank: at a fixed budget, paired with rank $r$ has the same parameter count as default with rank $r/2$, so paired effectively trades half the rotation degrees of freedom for doubled per-slice spectral capacity. Table~\ref{tab:share_R_ablation} reports this matched-budget comparison on LLaMA-2-7B math. Paired sharing wins both metrics at both budgets: at $\sim 27.6$\,M parameters, paired $r{=}128$ exceeds default $r{=}64$ by $2.3$ points on GSM8K and $1.5$ on MATH; at $\sim 13.6$\,M, paired $r{=}64$ exceeds default $r{=}32$ by $2.0$ on GSM8K and $0.5$ on MATH. We adopt no sharing as the default in our main results; the paired variant is a drop-in upgrade when doubling the per-slice rank within a fixed budget is preferred.

\begin{table}[h]
\centering
\caption{Shared spectrum scale ablation on LLaMA-2-7B, commonsense 8-task average.}
\label{tab:sub_variant_ablation}
\small
\setlength{\tabcolsep}{4pt}
\begin{tabular}{lcc}
\toprule
CORA tier & w/o $\boldsymbol{\delta}$ & w/ $\boldsymbol{\delta}$ \\
\midrule
$r{=}64$ (27.6\,M)  & 77.57 & 79.95 \\
$r{=}32$ (13.6\,M)  & 81.23 & 82.06 \\
$r{=}16$ ( 6.6\,M)  & 81.38 & \textbf{82.16} \\
$r{=}8$  ( 3.1\,M)  & 68.00 & 81.79 \\
\bottomrule
\end{tabular}
\end{table}

\begin{table}[h]
\centering
\caption{Rotation sharing across slice pairs on LLaMA-2-7B math (GSM8K + MATH) at matched parameter budgets. Paired at rank $r$ has the same parameter count as default at rank $r/2$. Each cell reports GSM8K/MATH; $\Delta$ is paired minus default.}
\label{tab:share_R_ablation}
\small
\setlength{\tabcolsep}{4pt}
\begin{tabular}{lccc}
\toprule
Params & default & paired & $\Delta$ \\
\midrule
$\sim 27.6$\,M & 62.0/12.7 ($r{=}64$)  & 64.3/14.2 ($r{=}128$) & $+2.3/+1.5$ \\
$\sim 13.6$\,M & 58.3/11.5 ($r{=}32$)  & 60.3/12.0 ($r{=}64$)  & $+2.0/+0.5$ \\
\bottomrule
\end{tabular}
\end{table}

\paragraph{Source-matrix variant ($W_{\mathrm{lr}}$ vs $W_0$).}
Theorem~\ref{thm:slice_coherent} applies to any source matrix with a per-slice SVD. We compare the rank-$r$ reconstruction $W_{\mathrm{lr}}$ (default) against the full pretrained weight $W_0$ on commonsense (Table~\ref{tab:source_matrix_ablation}). On both LLaMA models, $W_0$ at $r{=}128$ ($\sim 56$\,M) approaches the $W_{\mathrm{lr}}$ tier band (within $1$ point on LLaMA-2-7B, within band on LLaMA-3-8B) at $4$--$8\times$ the parameter cost. The rank-$r$ truncation in $W_{\mathrm{lr}}$ acts as a spectral regularizer in CORA's target regime; $W_0$ retains the tail singular directions of $W_0$ that the truncation discards and applies on tasks that rely on those directions (e.g., MATH).

\paragraph{$U$-side rotation only.}
Proposition~\ref{prop:R_equiv_shareQ} shows that CORA's default form realizes $Q_{U,i} = Q_{V,i} = Q_i$ from a single learnable $R_i$ via the identity $Q_i = U_i^\top R_i U_i$, with $R_i U_i = U_i Q_{U,i}$ acting on the $U$-side and $Q_i^\top V_i^\top$ rotating the $V$-side. To verify that the $V$-side rotation $Q_{V,i}$ is necessary, we drop it and keep only $Q_{U,i}$: the $U$-only variant uses $\widetilde W_i = R_i U_i \Sigma_i\,\mathrm{diag}(\mathbf{1}+\boldsymbol{\delta})\, V_i^\top$ in place of the default $\widetilde W_i = R_i U_i \Sigma_i\,\mathrm{diag}(\mathbf{1}+\boldsymbol{\delta})\, Q_i^\top V_i^\top$; the skew-parameter count of $R_i$ is unchanged. Table~\ref{tab:one_sided_ablation} reports HumanEval Pass@1 on Mistral-7B. Removing $Q_{V,i}$ costs $2.4$ to $9.8$ points across ranks; the loss is largest at $r{=}128$ where the F variant relies most on the two-sided structure. This confirms that realizing the coherent rotation form on both singular bases is a load-bearing part of CORA's design.

\begin{table}[h]
\begin{minipage}[t]{0.48\linewidth}
\centering
\caption{Source-matrix variant on commonsense 8-task average. Two variants: $W_{\mathrm{lr}}$ (default) and $W_0$. L2-7B / L3-8B: averages on LLaMA-2-7B / LLaMA-3-8B.}
\label{tab:source_matrix_ablation}
\small
\setlength{\tabcolsep}{4pt}
\begin{tabular}{lcrcc}
\toprule
Variant & $r$ & \#Params & L2-7B & L3-8B \\
\midrule
$W_{\mathrm{lr}}$ & 16 &  6.6\,M & 82.16 & 86.65 \\
$W_{\mathrm{lr}}$ & 32 & 13.6\,M & 82.06 & 86.75 \\
$W_0$          & 128 & 55.7\,M & 81.30 & 86.70 \\
\bottomrule
\end{tabular}
\end{minipage}
\hfill
\begin{minipage}[t]{0.48\linewidth}
\centering
\caption{$U$-side rotation only on Mistral-7B HumanEval Pass@1. \emph{default}: two-sided + $\boldsymbol{\delta}$; \emph{$U$-only + $\boldsymbol{\delta}$}: $V$-side dropped; \emph{$U$-only}: $V$-side and $\boldsymbol{\delta}$ dropped.}
\label{tab:one_sided_ablation}
\small
\setlength{\tabcolsep}{4pt}
\begin{tabular}{lccc}
\toprule
$r$ & default & $U$-only + $\boldsymbol{\delta}$ & $U$-only \\
\midrule
$r{=}128$ ($W_0$) & 40.9 & 31.1 & 29.3 \\
$r{=}64$  & 40.9 & 35.4 & 38.4 \\
$r{=}32$  & 40.9 & 38.4 & 40.9 \\
$r{=}16$  & 40.2 & 37.8 & 35.4 \\
$r{=}8$   & 40.2 & 36.0 & 39.6 \\
\bottomrule
\end{tabular}
\end{minipage}
\end{table}

\section{Conclusion}\label{sec:discussion}

We extended the coherent rotation form of finetuning from the full weight matrix to its row slices, and showed that the resulting per-slice condition reduces to an algebraic identity via $Q_i = U_i^\top R_i U_i$, removing the need for an explicit orthogonality constraint during optimization. CORA learns one orthogonal $R_i$ per slice of the rank-$r$ reconstruction $W_{\mathrm{lr}}$, uses $\tfrac{1}{2}m(r{-}1)$ trainable parameters per linear layer independent of the column dimension $k$, and reassembles into a dense weight of the same shape as $W_0$ at deployment. Empirically, CORA $r{=}16$ at $6.6$\,M parameters reaches $82.16\%$ commonsense accuracy on LLaMA-2-7B; the same method transfers to LLaMA-3-8B (commonsense and code) and Mistral-7B (code).

\paragraph{Limitations.}
CORA has several limitations. First, every adapted layer requires an offline SVD of $W_0$, and the per-slice factors $(U_i, \Sigma_i, V_i)$ must be computed and cached on disk before training begins. Second, the slice-wise forward $\widetilde W_i = R_i U_i \Sigma_i Q_i^\top V_i^\top$ adds matmuls per layer at training time relative to LoRA's $W_0 + BA$ form. Third, our largest evaluated backbone is 8B (LLaMA-3-8B), and we have not validated CORA at $\geq 30$B scale (e.g., LLaMA-3-70B, Mixtral $8\!\times\!22$B). Fourth, evaluation covers commonsense reasoning, math, and code generation under instruction tuning; vision-language, multilingual, long-context, and continual-learning settings remain unexplored.



{\small
\bibliographystyle{unsrtnat}
\bibliography{8-references}

@inproceedings{qiu2025oftv2,
  title     = {Orthogonal Finetuning Made Scalable},
  author    = {Qiu, Zeju and Liu, Weiyang and Weller, Adrian and Sch{\"o}lkopf, Bernhard},
  booktitle = {EMNLP},
  year      = {2025},
}

@article{gurses2025diablo,
  title   = {{DiaBlo}: Diagonal Blocks Are Sufficient for Finetuning},
  author  = {Gurses, Selcuk and Zhang, Aozhong and Deng, Yanxia and Dong, Xun and Li, Xin and Wang, Naigang and Yin, Penghang and Yang, Zi},
  journal = {arXiv preprint arXiv:2506.03230},
  year    = {2025},
}

@inproceedings{hu2022lora,
  title     = {{LoRA}: Low-Rank Adaptation of Large Language Models},
  author    = {Hu, Edward J. and Shen, Yelong and Wallis, Phillip and Allen-Zhu, Zeyuan and Li, Yuanzhi and Wang, Shean and Wang, Lu and Chen, Weizhu},
  booktitle = {ICLR},
  year      = {2022},
}

@inproceedings{hu2023llmadapters,
  title     = {{LLM-Adapters}: An Adapter Family for Parameter-Efficient Fine-Tuning of Large Language Models},
  author    = {Hu, Zhiqiang and Wang, Lei and Lan, Yihuai and Xu, Wanyu and Lim, Ee-Peng and Bing, Lidong and Xu, Xing and Poria, Soujanya and Lee, Roy Ka-Wei},
  booktitle = {EMNLP},
  year      = {2023},
}

@inproceedings{wang2025kasa,
  title     = {{KaSA}: Knowledge-Aware Singular-Value Adaptation of Large Language Models},
  author    = {Wang, Fan and Jiang, Juyong and Park, Chansung and Kim, Sunghun and Tang, Jing},
  booktitle = {ICLR},
  year      = {2025},
}

@inproceedings{lingam2024svft,
  title     = {{SVFT}: Parameter-Efficient Fine-Tuning with Singular Vectors},
  author    = {Lingam, Vijay and Tejaswi, Atula and Vavre, Aditya and Shetty, Aneesh and Gudur, Gautham Krishna and Ghosh, Joydeep and Dimakis, Alex and Choi, Eunsol and Bojchevski, Aleksandar and Sanghavi, Sujay},
  booktitle = {NeurIPS},
  year      = {2024},
}

@inproceedings{liu2024dora,
  title     = {{DoRA}: Weight-Decomposed Low-Rank Adaptation},
  author    = {Liu, Shih-Yang and Wang, Chien-Yi and Yin, Hongxu and Molchanov, Pavlo and Wang, Yu-Chiang Frank and Cheng, Kwang-Ting and Chen, Min-Hung},
  booktitle = {ICML},
  year      = {2024},
}

@inproceedings{liu2024boft,
  title     = {Parameter-Efficient Orthogonal Finetuning via Butterfly Factorization},
  author    = {Liu, Weiyang and Qiu, Zeju and Feng, Yao and Xiu, Yuliang and Xue, Yuxuan and Yu, Longhui and Feng, Haiwen and Liu, Zhen and Heo, Juyeon and Peng, Songyou and Wen, Yandong and Black, Michael J. and Weller, Adrian and Sch{\"o}lkopf, Bernhard},
  booktitle = {ICLR},
  year      = {2024},
}

@inproceedings{meng2024pissa,
  title     = {{PiSSA}: Principal Singular Values and Singular Vectors Adaptation of Large Language Models},
  author    = {Meng, Fanxu and Wang, Zhaohui and Zhang, Muhan},
  booktitle = {NeurIPS},
  year      = {2024},
}

@inproceedings{qiu2025poet,
  title     = {Reparameterized {LLM} Training via Orthogonal Equivalence Transformation},
  author    = {Qiu, Zeju and Buchholz, Simon and Xiao, Tim Z. and Dax, Maximilian and Sch{\"o}lkopf, Bernhard and Liu, Weiyang},
  booktitle = {NeurIPS},
  year      = {2025},
}

@inproceedings{qiu2023controlling,
  title     = {Controlling Text-to-Image Diffusion by Orthogonal Finetuning},
  author    = {Qiu, Zeju and Liu, Weiyang and Feng, Haiwen and Xue, Yuxuan and Feng, Yao and Liu, Zhen and Zhang, Dan and Weller, Adrian and Sch{\"o}lkopf, Bernhard},
  booktitle = {NeurIPS},
  year      = {2023},
}

@book{stewart1990matrix,
  title     = {Matrix Perturbation Theory},
  author    = {Stewart, G. W. and Sun, Ji-Guang},
  publisher = {Academic Press},
  year      = {1990},
}

@inproceedings{wang2025milora,
  title     = {{MiLoRA}: Harnessing Minor Singular Components for Parameter-Efficient {LLM} Finetuning},
  author    = {Wang, Hanqing and Li, Yixia and Wang, Shuo and Chen, Guanhua and Chen, Yun},
  booktitle = {NAACL},
  year      = {2025},
}

@article{wang2025wsbm,
  title         = {Geometric and Spectral Alignment for Deep Neural Network {II}},
  author        = {Liu, Ziran and Wang, Wei and Wang, Jinhao and
                   Wang, Pengcheng and Sui, Xinyi and Ruan, Cihan and
                   Ling, Nam and Jiang, Wei},
  journal       = {arXiv preprint arXiv:2605.02111},
  year          = {2026},
  eprint        = {2605.02111},
  archivePrefix = {arXiv},
  primaryClass  = {cs.LG},
  url           = {https://arxiv.org/abs/2605.02111},
}

@inproceedings{wu2025psoft,
  title     = {Efficient Orthogonal Fine-Tuning with Principal Subspace Adaptation},
  author    = {Wu, Fei and Hu, Jia and Min, Geyong and Wang, Shiqiang},
  booktitle = {ICLR},
  year      = {2026},
}

@inproceedings{zhang2023adalora,
  title     = {{AdaLoRA}: Adaptive Budget Allocation for Parameter-Efficient Fine-Tuning},
  author    = {Zhang, Qingru and Chen, Minshuo and Bukharin, Alexander and He, Pengcheng and Cheng, Yu and Chen, Weizhu and Zhao, Tuo},
  booktitle = {ICLR},
  year      = {2023},
}

@inproceedings{yang2024corda,
  title     = {{CorDA}: Context-Oriented Decomposition Adaptation of Large Language Models for Task-Aware Parameter-Efficient Fine-Tuning},
  author    = {Yang, Yibo and Li, Xiaojie and Zhou, Zhongzhu and Song, Shuaiwen Leon and Wu, Jianlong and Nie, Liqiang and Ghanem, Bernard},
  booktitle = {NeurIPS},
  year      = {2024},
}

@inproceedings{kopiczko2024vera,
  title     = {{VeRA}: Vector-based Random Matrix Adaptation},
  author    = {Kopiczko, Dawid J. and Blankevoort, Tijmen and Asano, Yuki M.},
  booktitle = {ICLR},
  year      = {2024},
}

@inproceedings{li2024loftq,
  title     = {{LoftQ}: {LoRA}-Fine-Tuning-Aware Quantization for Large Language Models},
  author    = {Li, Yixiao and Yu, Yifan and Liang, Chen and He, Pengcheng and Karampatziakis, Nikos and Chen, Weizhu and Zhao, Tuo},
  booktitle = {ICLR},
  year      = {2024},
}

@inproceedings{wang2024loraga,
  title     = {{LoRA-GA}: Low-Rank Adaptation with Gradient Approximation},
  author    = {Wang, Shaowen and Yu, Linxi and Li, Jian},
  booktitle = {NeurIPS},
  year      = {2024},
}

@inproceedings{wang2025lorapro,
  title     = {{LoRA-Pro}: Are Low-Rank Adapters Properly Optimized?},
  author    = {Wang, Zhengbo and Liang, Jian and He, Ran and Wang, Zilei and Tan, Tieniu},
  booktitle = {ICLR},
  year      = {2025},
}

@inproceedings{salimans2016weight,
  title     = {Weight Normalization: A Simple Reparameterization to Accelerate Training of Deep Neural Networks},
  author    = {Salimans, Tim and Kingma, Diederik P.},
  booktitle = {NeurIPS},
  year      = {2016},
}

@inproceedings{zhang2025lori,
  title     = {{LoRI}: Reducing Cross-Task Interference in Multi-Task Low-Rank Adaptation},
  author    = {Zhang, Juzheng and You, Jiacheng and Panda, Ashwinee and Goldstein, Tom},
  booktitle = {COLM},
  year      = {2025},
}

@article{yu2024metamath,
  title     = {{MetaMath}: Bootstrap Your Own Mathematical Questions for Large Language Models},
  author    = {Yu, Longhui and Jiang, Weisen and Shi, Han and Yu, Jincheng and Liu, Zhengying and Zhang, Yu and Kwok, James T. and Li, Zhenguo and Weller, Adrian and Liu, Weiyang},
  journal   = {arXiv preprint arXiv:2309.12284},
  year      = {2023},
}

@article{cobbe2021gsm8k,
  title     = {Training Verifiers to Solve Math Word Problems},
  author    = {Cobbe, Karl and Kosaraju, Vineet and Bavarian, Mohammad and Chen, Mark and Jun, Heewoo and Kaiser, Lukasz and Plappert, Matthias and Tworek, Jerry and Hilton, Jacob and Nakano, Reiichiro and Hesse, Christopher and Schulman, John},
  journal   = {arXiv preprint arXiv:2110.14168},
  year      = {2021},
}

@article{hendrycks2021math,
  title     = {Measuring Mathematical Problem Solving with the {MATH} Dataset},
  author    = {Hendrycks, Dan and Burns, Collin and Kadavath, Saurav and Arora, Akul and Basart, Steven and Tang, Eric and Song, Dawn and Steinhardt, Jacob},
  journal   = {arXiv preprint arXiv:2103.03874},
  year      = {2021},
}

@misc{chaudhary2023codealpaca,
  title        = {{Code Alpaca}: An Instruction-following {LLaMA} Model for Code Generation},
  author       = {Chaudhary, Sahil},
  year         = {2023},
  howpublished = {\url{https://github.com/sahil280114/codealpaca}},
}

@article{chen2021humaneval,
  title     = {Evaluating Large Language Models Trained on Code},
  author    = {Chen, Mark and Tworek, Jerry and Jun, Heewoo and Yuan, Qiming and Pinto, Henrique Ponde de Oliveira and Kaplan, Jared and Edwards, Harri and Burda, Yuri and Joseph, Nicholas and Brockman, Greg and others},
  journal   = {arXiv preprint arXiv:2107.03374},
  year      = {2021},
}

@book{horn2012matrix,
  title     = {Matrix Analysis},
  author    = {Horn, Roger A. and Johnson, Charles R.},
  edition   = {2nd},
  publisher = {Cambridge University Press},
  year      = {2012},
}
}

\appendix

\section{Proof of Theorem~\ref{thm:slice_coherent}}\label{app:proof_slice_coherent}
\begin{proof}
The Frobenius norm of a row-partitioned perturbation decomposes as
\[
    \|\Delta W\|_F^2
    =
    \sum\nolimits_{i=1}^{s}\|\Delta W_i\|_F^2 .
\]

By Assumption~\ref{assump:slice_regularity}, each slice
$W_i = U_i\Sigma_iV_i^\top$ satisfies the standard stability hypotheses
(non-degenerate spectrum, bounded condition number, small Frobenius
perturbation). The per-slice Frobenius norm is invariant under the SVD
coordinates:
\[
    \|\Delta W_i\|_F
    =
    \|U_i^\top \Delta W_i V_i\|_F ,
\]
since $U_i$ is orthogonal and $V_i$ has orthonormal columns on the slice
subspace. The slice-level minimum-perturbation problem therefore reduces,
in $(U_i,\Sigma_i,V_i)$ coordinates, to a spectrum-fixed Frobenius
minimization on the local core.

Write the slice perturbation in local coordinates as
\(\Delta W_i = U_i P_i V_i^\top\), for some $P_i \in \mathbb{R}^{r\times r}$. Then
\(W_i^\star
    =
    U_i(\Sigma_i+P_i)V_i^\top\).
Let
\(\Sigma_i+P_i
    =
    Q_{U,i}^\star D_i (Q_{V,i}^\star)^\top\)
be the SVD of the local adapted core, with $D_i$ diagonal. By the
Schur--Horn / Von Neumann trace-inequality argument
\cite[\S 4.3 and \S 7.4]{horn2012matrix}, the constrained minimum of
\(\|(\Sigma_i+P_i)-\Sigma_i\|_F\) over $P_i$, at fixed spectrum of
$\Sigma_i+P_i$, is attained by a unitary similarity transform of
$\Sigma_i$. Therefore,
\[
    Q_{U,i}^\star = Q_{V,i}^\star = Q_i,
    \qquad
    D_i = \Sigma_i+\Delta\Sigma_i .
\]
Substituting yields \(W_i^\star = U_i Q_i (\Sigma_i + \Delta\Sigma_i) Q_i^\top V_i^\top\).
\end{proof}

\section{Proof of Proposition~\ref{prop:R_equiv_shareQ}}\label{app:proof_one_sided}
\begin{proof}
Part (i) is the algebraic identity $\widetilde W_i = U_i Q_i \Sigma_i\,\mathrm{diag}(\mathbf{1}+\boldsymbol{\delta})\, Q_i^\top V_i^\top$ established in Lemma~\ref{lem:cora_coherent}; we focus on (ii). For brevity we present the argument for $\boldsymbol{\delta} = \mathbf{0}$; replacing $\Sigma_i$ throughout with the diagonal matrix $\Sigma_i\,\mathrm{diag}(\mathbf{1}+\boldsymbol{\delta})$ leaves every step unchanged, since the spectral and orthogonality structure used below is preserved.

Substituting $Q_i=U_i^\top R_iU_i$ into $\widetilde W_i = R_i U_i \Sigma_i Q_i^\top V_i^\top$ gives
\[
    \widetilde W_i
    =
    R_iU_i\Sigma_iU_i^\top R_i^\top U_iV_i^\top
    =
    \underbrace{R_i(U_i\Sigma_iU_i^\top)R_i^\top}_{B_i}
    \underbrace{U_iV_i^\top}_{C_i}.
\]
Since $U_i\Sigma_iU_i^\top$ is symmetric positive semidefinite, $B_i$ is also symmetric positive semidefinite. Moreover,
    \(C_iC_i^\top
    =
    U_iV_i^\top V_iU_i^\top
    =
    U_iU_i^\top
    =
    I_r\),
so $C_i$ has orthonormal rows.

Let $B_i=E_i\Lambda_iE_i^\top$, with $\Lambda_i = \diag(\mu_1, \ldots, \mu_r)$, $\mu_\ell \geq 0$. Then
    \(\widetilde W_i\widetilde W_i^\top
    =
    B_iC_iC_i^\top B_i^\top
    =
    B_i^2
    =
    E_i\Lambda_i^2E_i^\top\),
so the left singular vectors of $\widetilde W_i$ are $\widetilde U_i=E_i$. The corresponding right singular vectors are
    \(\widetilde V_i
    =
    C_i^\top E_i
    =
    V_iU_i^\top E_i\).
Therefore,
    \(Q_{U,i}^\star
    =
    U_i^\top \widetilde U_i
    =
    U_i^\top E_i\),
and
    \(Q_{V,i}^\star
    =
    V_i^\top \widetilde V_i
    =
    V_i^\top V_iU_i^\top E_i
    =
    U_i^\top E_i\).
Hence $Q_{U,i}^\star=Q_{V,i}^\star$, proving (ii).
\end{proof}

\section{Algebraic identity for CORA}\label{app:coherent_derivation}

The main text's Theorem~\ref{thm:global_coherent} (``coherent rotation form'') is the $s = 1$ specialization of Theorem~\ref{thm:slice_coherent}, whose per-slice proof above applies directly; the parameterization $W = U Q_U \Sigma Q_V^\top V^\top$ matches~\cite[Defn.~9.3]{wang2025wsbm}.
Direct calculation shows that for any orthogonal $R_i$, CORA's parameterization satisfies Eq.~\eqref{eq:coherent_form} exactly.
The argument uses no optimization assumption and verifies the method in Section~\ref{sec:method_overview}.

\begin{lemma}[Coherent rotation form holds identically]\label{lem:cora_coherent}
Let $W_i = U_i \Sigma_i V_i^\top$ be a slice with SVD, $U_i \in \R^{r \times r}$ square orthogonal, $\Sigma_i \in \R^{r \times r}$ diagonal, $V_i \in \R^{k \times r}$ with orthonormal columns. Let $R_i \in \R^{r \times r}$ be orthogonal, and set $Q_i := U_i^\top R_i U_i$. Then:
\begin{enumerate}
    \item[(i)] $Q_i$ is orthogonal.
    \item[(ii)] The CORA parameterization $\widetilde W_i := R_i U_i \Sigma_i Q_i^\top V_i^\top$ equals
    \[
        \widetilde W_i \;=\; U_i \, Q_i \, \Sigma_i \, Q_i^\top \, V_i^\top,
    \]
    i.e., the coherent rotation form~\eqref{eq:coherent_form} with $\Delta\Sigma = 0$.
    \item[(iii)] Allowing a per-layer multiplicative scale $\Sigma_i \mapsto \Sigma_i\,\mathrm{diag}(\mathbf{1}+\boldsymbol{\delta})$ with $\boldsymbol{\delta} \in \mathbb{R}^r$ (the \textbf{CORA+$\boldsymbol{\delta}$} variant) produces $\widetilde W_i = U_i Q_i \Sigma_i\,\mathrm{diag}(\mathbf{1}+\boldsymbol{\delta})\,Q_i^\top V_i^\top$. This still satisfies Eq.~\eqref{eq:coherent_form}, though the spectrum shift it spans is a one-parameter family rather than the full diagonal $\Delta\Sigma_i$ allowed by Theorem~\ref{thm:slice_coherent}.
\end{enumerate}
\end{lemma}

\begin{proof}
(i) Since $U_i$ is square orthogonal, $U_i U_i^\top = U_i^\top U_i = I_r$. Hence
$Q_i^\top Q_i = (U_i^\top R_i U_i)^\top (U_i^\top R_i U_i) = U_i^\top R_i^\top (U_i U_i^\top) R_i U_i = U_i^\top R_i^\top R_i U_i = U_i^\top U_i = I_r$.

(ii) Rewrite $R_i = U_i (U_i^\top R_i U_i) U_i^\top = U_i Q_i U_i^\top$ using $U_i U_i^\top = I_r$. Substituting,
\[
    \widetilde W_i \;=\; R_i U_i \Sigma_i Q_i^\top V_i^\top
    \;=\; (U_i Q_i U_i^\top)\, U_i \Sigma_i Q_i^\top V_i^\top
    \;=\; U_i Q_i (U_i^\top U_i) \Sigma_i Q_i^\top V_i^\top
    \;=\; U_i Q_i \Sigma_i Q_i^\top V_i^\top.
\]

(iii) Immediate from (ii) by replacing $\Sigma_i$ with $\Sigma_i\,\mathrm{diag}(\mathbf{1}+\boldsymbol{\delta})$ (a diagonal operation; per-layer across slices).
\end{proof}

Lemma~\ref{lem:cora_coherent} shows that a single rotation $R_i$ on the left singular basis is enough to recover the coherent rotation form; turning on the shared scale $\boldsymbol{\delta}$ adds a one-parameter spectrum shift inside Eq.~\eqref{eq:coherent_form}.

\section{Cayley implementation}\label{app:cayley_impl}

The map $\phi:\mathrm{skew}(b) \to \mathrm{O}(b)$ used in Section~\ref{sec:method_overview} converts each skew-symmetric block $A_i^{(g)}$ into an orthogonal rotation via the closed-form Cayley map $\phi(A) = (I - A)(I + A)^{-1}$, computed with a linear solve. It produces an exactly orthogonal rotation up to numerical precision and imposes no convergence-radius constraint on $A_i^{(g)}$, so no auxiliary regularizer is needed.

\section{Per-slice coherence diagnostic}\label{app:slice_coherence}

We measure per-slice coherence $\cos(Q_{U,i}, Q_{V,i})$ on $32 \times 32$ row slices, averaged over the 160 target layers of LLaMA-2-7B, to verify that adapters satisfying Theorem~\ref{thm:slice_coherent} keep this quantity close to 1. Full FT, PiSSA, MiLoRA, and CORA all stay at or above $0.995$, confirming that the SVD-family adapters realize the per-slice coherent rotation form.

\begin{table}[h]
\centering
\caption{Per-slice coherence $\cos(Q_{U,i}, Q_{V,i})$ at $32 \times 32$ row slices, averaged over 160 layers of LLaMA-2-7B. CORA satisfies $Q_{U,i} = Q_{V,i}$ identically by Lemma~\ref{lem:cora_coherent}.}
\label{tab:slice_coherence}
\small
\begin{tabular}{lc}
\toprule
Method & Per-slice $\cos(Q_U, Q_V)$ \\
\midrule
Full FT           & 0.9994 \\
PiSSA $r{=}128$   & 0.9995 \\
MiLoRA $r{=}32$   & 0.9950 \\
\midrule
CORA (ours)         & \textbf{0.9999} \\
\bottomrule
\end{tabular}
\end{table}

\section{Spectral-fingerprint diagnostic}\label{app:usv_diagnostic}

The two adaptation axes (spectrum shift $\Delta\Sigma$ and coherent rotation $Q$ with $Q_U = Q_V$) can be measured directly against a Full-FT reference.
This appendix defines the diagnostic, reports it across representative PEFT methods, and discusses what the numbers reveal about how closely each method satisfies Eq.~\eqref{eq:coherent_form}.

\paragraph{Two quantities.}
Given a finetuned weight $W$ with SVD $W = U \Sigma V^\top$, we measure, relative to $W_0$:
\begin{enumerate}[leftmargin=1.5em]
    \item[(i)] the relative singular-value shift $\|\Delta\Sigma / \Sigma\|$; and
    \item[(ii)] the spectral coherence $\cos(Q_U, Q_V)$ with $Q_U = U_0^\top U_{:,r}$ and $Q_V = V_0^\top V_{:,r}$.
\end{enumerate}
Quantity (i) measures \emph{how much} the adapter moves the spectrum; quantity (ii) measures whether the left and right subspaces move \emph{together}, which is the defining condition of the coherent rotation form.

\begin{table}[h]
\centering
\caption{Global-SVD spectral fingerprint of representative finetuning methods on LLaMA-2-7B (commonsense\_170k, 3 epochs, averaged over the 160 target layers).}
\label{tab:usv_empirical}
\small
\begin{tabular}{@{}lccc@{}}
\toprule
Method & $|\Delta\Sigma/\Sigma|$ & $\cos(Q_U, Q_V)$ & Avg acc.\ \\
\midrule
Full FT                       & $0.001$ & $0.99999$ & 83.5 \\
PiSSA $r{=}128$               & $0.006$ & $0.99993$ & 81.2 \\
DoRA $r{=}128$                & $0.391$ & $0.99074$ & 77.2 \\
\bottomrule
\end{tabular}
\end{table}

\paragraph{Coherence tracks accuracy.}
$\cos(Q_U, Q_V) \geq 0.9999$ holds for Full FT; an adapter that approximates the minimum-perturbation form (Theorem~\ref{thm:global_coherent}) inherits the same coherence. PiSSA at $r{=}128$ comes within $10^{-5}$ of Full FT on this axis and is within $2.3$ points in accuracy. DoRA at the same parameter budget drifts to $\cos(Q_U, Q_V) = 0.99074$ and trails by $\sim 6$ points. The diagnostic separates adapters that satisfy the coherent rotation form from those that depart from it; CORA, by Lemma~\ref{lem:cora_coherent}, satisfies Eq.~\eqref{eq:coherent_form} exactly.

\section{Hyperparameter summary}\label{app:hyperparams}

\begin{table}[h]
\centering
\caption{Shared hyperparameters across main-text experiments.}
\label{tab:hp_shared}
\small
\begin{tabular}{ll}
\toprule
Optimizer              & AdamW ($\beta_1 = 0.9, \beta_2 = 0.999$) \\
LR schedule            & Cosine decay, 100 warmup steps \\
Batch size             & 16 (commonsense / math), 32 (code) \\
Micro-batch (H100)     & 16, no gradient checkpointing \\
Micro-batch (L40S)     & 4 + gradient checkpointing \\
Weight decay           & 0 \\
Precision              & bf16 mixed \\
\texttt{torch.compile} & Enabled \\
$L_2$ on $\boldsymbol{\delta}$  & $10^{-3}$ \\
\bottomrule
\end{tabular}
\end{table}


\end{document}